\documentclass[a4paper,UTF8]{article}

\usepackage[margin=1.5in,a4paper]{geometry}
\usepackage{amsmath}
\usepackage{amsfonts}
\usepackage{amsthm}
\usepackage[ruled,vlined]{algorithm2e}
\usepackage{mathtools}
\usepackage{tabularx}
\usepackage{graphicx}
\usepackage{subfigure}
\usepackage{enumerate}
\usepackage{float}
\usepackage{url}
\usepackage{verbatim}
\usepackage[final]{hyperref}

\DeclareMathOperator*{\argmin}{arg\,min}
\DeclareMathOperator*{\trace}{Tr}
\newcommand{\tp}{^{\mathrm{T}}}
\newcommand{\invtp}{^{-\mathrm{T}}}
\newcommand{\grad}{\nabla}
\newcommand{\hessian}{\nabla^2}
\newcommand{\rbrac}[1]{({#1})}
\newcommand{\rBrac}[1]{\left({#1}\right)}

\newcommand{\cBrac}[1]{\left\{{#1}\right\}}
\newcommand{\norm}[1]{\Vert{#1}\Vert}
\newcommand{\Norm}[1]{\left\Vert{#1}\right\Vert}
\newcommand{\abs}[1]{\vert{#1}\vert}

\newtheorem{assumption}{Assumption}

\newtheorem{lemma}{Lemma}
\newtheorem{theorem}{Theorem}

\title{\textbf{Detailed Proofs of \\ \textit{Alternating Minimization Based Trajectory Generation for Quadrotor Aggressive Flight}}}
\author{Zhepei Wang, Xin Zhou, Chao Xu and Fei Gao}
\date{}

\begin{document}
    \maketitle

\begin{abstract}
    This technical report provides detailed theoretical analysis of the algorithm used in \textit{Alternating Minimization Based Trajectory Generation for Quadrotor Aggressive Flight}. An assumption is provided to ensure that settings for the objective function are meaningful. What's more, we explore the structure of the optimization problem and analyze the global/local convergence rate of the employed algorithm.
\end{abstract}

\section{Preliminaries}
Piece-wise polynomial representation of trajectory is adopted. Any segment of the trajectory can be denoted as an $N$-order polynomial
\begin{equation}
p(t)=\mathbf{c}\tp \beta(t), t\in[0, T]
\end{equation}
where $\mathbf{c}\in\mathbb{R}^{(N+1)\times{3}}$ is the coefficient matrix, $T$ is the duration and
\begin{equation}
\beta(t)=\rBrac{1, t, t^2, \cdots, t^N}\tp
\end{equation}
 is a basis function. It is worth noting that $N$ should be an odd number hereafter, which makes the mapping bijective between the coefficient matrix and its boundary condition.

Consider derivatives of $p(t)$ up to $(N-1)/2$ order
\begin{equation}
\mathbf{d}(t)=\rBrac{p(t), \dot{p}(t), \cdots, p^{\rBrac{\frac{N-1}{2}}}(t)}\tp,
\end{equation}
we have
\begin{equation}
\mathbf{d}(t)=\mathbf{B}(t)\mathbf{c}
\end{equation}
where
\begin{equation}
\mathbf{B}(t)=\rBrac{\beta(t), \dot{\beta}(t), \cdots, \beta^{\rBrac{\frac{N-1}{2}}}(t)}\tp.
\end{equation}
We denote $\mathbf{d}_{start}$ and $\mathbf{d}_{end}$ by $\mathbf{d}(0)$ and $\mathbf{d}(T)$, respectively. The boundary condition is described by tuple $\rBrac{\mathbf{d}\tp_{start}, \mathbf{d}\tp_{end}}\tp$. The following mapping holds:
\begin{equation}
\label{eq:RepresentationMapping}
\rBrac{\mathbf{d}\tp_{start}, \mathbf{d}\tp_{end}}\tp=\mathbf{A}(T)\mathbf{c}
\end{equation}
where
\begin{equation}
\mathbf{A}(T)=\rBrac{\mathbf{B}\tp(0),\mathbf{B}\tp(T)}\tp
\end{equation}
is the mapping matrix. Since $N$ is an odd number, it is easy to know that $\mathbf{A}(T)$ is a non-singular square matrix. In other words, the mapping in \ref{eq:RepresentationMapping} is bijective. Therefore, any segment of a trajectory can be equivalently expressed by tuple $\rBrac{\mathbf{d}_{start}, \mathbf{d}_{end}, T}$ or tuple $\rBrac{\mathbf{c}, T}$.

Consequently, we consider an $M$-segment trajectory $\mathbf{P}$ parametrized by time allocation $\mathbf{T}=\rBrac{T_1, T_2, \cdots, T_M}\tp$ as well as boundary conditions $\mathbf{D}=\rBrac{d\tp_1, d\tp_2, \cdots, d\tp_{M+1}}\tp$ of all segments. The trajectory is defined by
\begin{equation}
\mathbf{P}(t):=\mathbf{d}_m\tp\mathbf{A}(T_m)\invtp\beta( t-\sum_{i=1}^{m-1}{T_i})
\end{equation}
where $t$ lies in the $m$-th segment and $\mathbf{d}_m=\rBrac{d\tp_{m}, d\tp_{m+1}}\tp$ is a boundary condition of the $m$-th segment. Normally, some entries in $\mathbf{D}$ are fixed while the others are to be optimized. We split $\mathbf{D}$ into two parts, the fixed part $\mathbf{D}_F$ which is viewed as constant, and the free part $\mathbf{D}_P$ which is to be optimized.
Then, the whole trajectory can be fully determined by
\begin{equation}
\mathbf{P}=\mathbf{\Phi}(\mathbf{D}_P, \mathbf{T}).
\end{equation}

\section{Optimization Objective}
The following time regularized quadratic objective function is used:
\begin{equation}
\label{eq:IntegralFormObjective}
J(\mathbf{P})=\int_{0}^{\sum_{m=1}^{M}{T_m}}\rBrac{{\rho + \sum_{i=D_{min}}^{D_{max}}{w_i\Norm{\mathbf{P}^{(i)}(t)}^2}}}\mathrm{d}t
\end{equation}
where $D_{min}$ and $D_{max}$ are the lowest and the highest order of derivative to be penalized respectively, $w_i$ is the weight of the $i$-order derivative and $\rho$ is the weight of time regularization. When $D_{max} > (N-1)/2$, some derivatives on the boundary of each segment may not exist, hence we sum up objectives on all segments instead, which have the form
\begin{equation}
\label{eq:KthSegmentObjective}
J_m(\mathbf{d}_m, T_m):=\rho T_m+\trace\cBrac{\mathbf{d}_m\tp\mathbf{A}(T_m)\invtp \mathbf{Q}(T_m)\mathbf{A}(T_m)^{-1}\mathbf{d}_m}
\end{equation}
for the $m$-th segment, where $\mathbf{Q}(T_m)$ is a symmetric matrix \cite{Bry2015AggressiveFO} consisting of high powers of $T_m$, and $\trace\cBrac{\cdot}$ is trace operation. The overall objective $J(\mathbf{D}_P, \mathbf{T}):=J(\mathbf{\Phi}(\mathbf{D}_P, \mathbf{T}))$ is formulated as
\begin{equation}
\label{eq:ObjectiveStructure}
J(\mathbf{D}_P, \mathbf{T})=\rho\Norm{\mathbf{T}}_1+\trace\cBrac{\begin{pmatrix}\mathbf{D}_F\\\mathbf{D}_P\end{pmatrix}\tp\mathbf{C}\tp\mathbf{H}(\mathbf{T})\mathbf{C}\begin{pmatrix}\mathbf{D}_F\\\mathbf{D}_P\end{pmatrix}}
\end{equation}
\begin{equation}
\mathbf{H}(\mathbf{T})=\bigoplus_{m=1}^{M}{\mathbf{A}(T_m)\invtp\mathbf{Q}(T_m)\mathbf{A}(T_m)^{-1}}
\end{equation}
where $\mathbf{H}(\mathbf{T})$ is the direct sum of its $M$ diagonal blocks, and $\mathbf{C}$ is a permutation matrix.

In Eq.~\ref{eq:ObjectiveStructure}, $N, M, D_{min}, D_{max}, \rho, w_i, \mathbf{D}_F$ and $\mathbf{C}$ are all parameters that directly determine the structure of $J(\mathbf{D}_P, \mathbf{T})$. It is important to know that not all settings for $J$ are legal. Instead of restricting those parameters, we make the following assumption on the objective function such that the setting is meaningful.
\begin{assumption}
    \label{asm:LegalityCondition}
    For any finite $\alpha$, the corresponding $\alpha$-sublevel set of $J$
    \begin{equation}
    \mathit{L}_{\alpha}^{-}(J):=\cBrac{(\mathbf{D}_P, \mathbf{T})~\Big|~J(\mathbf{D}_P, \mathbf{T})\leq{\alpha}}
    \end{equation}
    is bounded and satisfies
    \begin{equation}
    \label{eq:StrictPositiveness}
    \mathit{L}_{\alpha}^{-}(J)\subset\cBrac{(\mathbf{D}_P, \mathbf{T})~\Big|~\mathbf{T}\in\mathbb{R}_+^M}.
    \end{equation}
\end{assumption}
Intuitively, Assumption \ref{asm:LegalityCondition} forbids the objective from taking meaningful value when decision variables are extremely large or any duration is extremely small. For example, consecutive repeating waypoints with identical boundary conditions fixed in $\mathbf{D}_F$ are illegal, because  the optimal duration on corresponding segment becomes $0$ which violates condition (\ref{eq:StrictPositiveness}). In other words, the segment should not exist if the objective is to be minimized. Another example is that $\rho\leq{0}$ is also illegal. Non-positive weight on total duration means that the objective can be sufficiently low when duration on each segment is large enough. In such a case, the boundness condition is violated.

\section{Unconstrained Optimization Algorithm}
\begin{algorithm}
    \caption{Unconstrained Spatial-Temporal AM}
    \label{alg:UnconstrainedSpatialTemporalAM}
    \KwIn{$\mathbf{D}_P^0, K\in\mathbb{Z}_+, \delta>0$}
    \KwOut{$\mathbf{D}_P^*, \mathbf{T}^*$}
    \Begin
    {
        $\mathbf{T}^0 \leftarrow \argmin_{\mathbf{T}}{J(\mathbf{D}_P^0, \mathbf{T})}$\;
        $J_l \leftarrow J(\mathbf{D}_P^0, \mathbf{T}^0), k \leftarrow 0$\;
        \While{$k<K$}
        {
            $\mathbf{D}_P^{k+1} \leftarrow \argmin_{\mathbf{D}_P}{J(\mathbf{D}_P, \mathbf{T}^{k})}$\;
            $\mathbf{T}^{k+1} \leftarrow \argmin_{\mathbf{T}}{J(\mathbf{D}_P^{k+1}, \mathbf{T})}$\;
            $J_{c} \leftarrow J(\mathbf{D}_P^{k+1}, \mathbf{T}^{k+1})$\;
            \If{$\abs{J_l-J_c}<\delta$}{\textbf{break}}
            $J_l \leftarrow J_c, k \leftarrow k+1$\;
        }
        $\mathbf{D}_P^* \leftarrow \mathbf{D}_P^{k}, \mathbf{T}^* \leftarrow \mathbf{T}^{k}$\;

        \Return{$\mathbf{D}_P^*, \mathbf{T}^*$};
    }
\end{algorithm}
To optimize Eq.~\ref{eq:ObjectiveStructure}, Algorithm.~\ref{alg:UnconstrainedSpatialTemporalAM} is proposed.
Initially, $\mathbf{T}^{0}$ is solved for any provided $\mathbf{D}_P^0$. After that, the minimization of the objective function is done through a two-phase process, in which only one of $\mathbf{D}_P$ and $\mathbf{T}$ is optimized while the other is fixed.

In the first phase, the sub-problem
\begin{equation}
\label{eq:OptimizeD}
\mathbf{D}_P^*(\mathbf{T})=\argmin_{\mathbf{D}_P}{J(\mathbf{D}_P, \mathbf{T})}
\end{equation}
is solved for each $\mathbf{T}^k$. We employ the unconstrained QP formulation by Richter et al. \cite{Bry2015AggressiveFO}, which we briefly introduce here. The matrix $\mathbf{R}(\mathbf{T})=\mathbf{C}\tp\mathbf{H}(\mathbf{T})\mathbf{C}$ is partitioned as
\begin{equation}
\mathbf{R}(\mathbf{T})=\begin{pmatrix}\mathbf{R}_{FF}(\mathbf{T}) & \mathbf{R}_{FP}(\mathbf{T}) \\
\mathbf{R}_{PF}(\mathbf{T}) & \mathbf{R}_{PP}(\mathbf{T})\end{pmatrix}.
\end{equation}
then the solution is obtained analytically through
\begin{equation}
\mathbf{D}_P^*(\mathbf{T})=-\mathbf{R}_{PP}(\mathbf{T})^{-1}\mathbf{R}_{FP}(\mathbf{T})\mathbf{D}_F.
\end{equation}

In the second phase, the sub-problem
\begin{equation}
\label{eq:OptimizeT}
\mathbf{T}^*(\mathbf{D}_P)=\argmin_{\mathbf{T}}{J(\mathbf{D}_P, \mathbf{T})}
\end{equation}
is solved for each $\mathbf{D}_P^k$. In this phase, the scale of sub-problem can be reduced into each segment.
Due to our representation of trajectory, once $\mathbf{D}_P$ is fixed, the boundary conditions $\mathbf{D}$ isolate each entry in $\mathbf{T}$ from the others. Therefore, $T_m$ can be optimized individually to get all entries of $\mathbf{T}^*(\mathbf{D}_P)$. As for the $m$-th segment, its cost $J_m$ in (\ref{eq:KthSegmentObjective}) is indeed a rational function of $T_m$. We show the structure of $J_m$ and omit the trivial deduction for brevity:
\begin{equation}
\label{eq:RationalCostOnSegment}
J_m(T)=\rho T+\frac{1}{T^{p_{n}}}\sum\limits_{i=0}^{p_{d}}{\alpha_i T^i}
\end{equation}
where $p_{n}=2D_{max}-1$ and $p_{d}=2(D_{max}-D_{min})+N-1$ are orders of numerator and denominator respectively. The coefficient $\alpha_i$ is determined by $\mathbf{d}_m$. It is clear that $J_m(T)$ is smooth on $T\in\rbrac{0, +\infty}$. Due to the positiveness of $J_m(T)$, we have $J_m(T)\rightarrow+\infty$ as $T\rightarrow+\infty$ or $T\rightarrow0^+$. Therefore, the minimizer exists for
\begin{equation}
T_m^*(\mathbf{D}_P)=\argmin_{T\in\rbrac{0, +\infty}}{J_m(T)}.
\end{equation}
To find all candidates, we compute the derivative of (\ref{eq:RationalCostOnSegment}):
\begin{equation}
\frac{\mathrm{d}J_m(T)}{\mathrm{d}T}=\rho+\frac{1}{T^{1+p_{n}}}\sum\limits_{i=0}^{p_{d}}{(i-p_n)\alpha_i T^i}.
\end{equation}
The minimum exists in the solution set of ${\mathrm{d}J_m(T)}/{\mathrm{d}T}=0$, which can be calculated through any modern univariate polynomial real-roots solver \cite{Sagraloff2013ComputingRR}. The second phase is completed by updating every entry $T_m^*(\mathbf{D}_P)$ in $\mathbf{T}^*(\mathbf{D}_P)$.

\section{Convergence Analysis}
We first explore some basic properties of $J(\mathbf{D}_P, \mathbf{T})$, which help a lot in convergence analysis of Algorithm~\ref{alg:UnconstrainedSpatialTemporalAM}. We have already shown that $J(\mathbf{D}_P, \mathbf{T})$ are rational function of each entry in $\mathbf{T}$. As for the $\mathbf{D}_P$ part, it is indeed partially convex, which is given by the following lemma.
\begin{lemma}
    \label{lm:PartialConvexity}
    $J(\mathbf{D}_P, \mathbf{T})$ is convex in $\mathbf{D}_P$ for any $\mathbf{T}\in\mathbb{R}_+^M$, provided that Assumption~\ref{asm:LegalityCondition} holds.
\end{lemma}
\begin{proof}
    Assumption~\ref{asm:LegalityCondition} implies that $\rho>{0}$, $w_i\geq{0}$ holds for all $D_{min}\leq{i}\leq{D_{max}}$ and at least one $w_i$ is nonzero. Otherwise, the boundness on $\mathit{L}_{\alpha}^{-}(J)$ or positiveness on its time allocation is violated. Thus, for any $\mathbf{T}\in\mathbb{R}_+^M$, the objective function is always positive, which can be seen from (\ref{eq:IntegralFormObjective}). The non-negativity of $J(\mathbf{D}_P, \mathbf{T})$ implies the positive semidefiniteness of the symmetric matrix $\mathbf{R}(\mathbf{T})$. Since $\mathbf{R}_{PP}(\mathbf{T})$ is the principal submatrix of $\mathbf{R}(\mathbf{T})$, it is also positive semidefinite. We compute the Hessian matrix of $J(\mathbf{D}_P, \mathbf{T})$ with respect to $\mathbf{D}_P$:
    \begin{equation}
    \hessian_{\mathbf{D}_P}{J(\mathbf{D}_P, \mathbf{T})}=2\bigoplus_{k=1}^{3}{\mathbf{R}_{PP}(\mathbf{T})}
    \end{equation}
    which means $\hessian_{\mathbf{D}_P}{J(\mathbf{D}_P, \mathbf{T})}$ is positive semidefinite. Therefore, $J(\mathbf{D}_P, \mathbf{T})$ is convex in $\mathbf{D}_P$.
\end{proof}

\begin{lemma}
    \label{lm:LGradientConvexProperty}
    For any convex function $f:\mathbb{R}^{I\times{J}}\mapsto\mathbb{R}$, if the following inequality holds for any $\mathbf{X}, \mathbf{Y}\in\mathbb{R}^{I\times{J}}$
    \begin{equation}
    \Norm{\grad{f(\mathbf{X})}-\grad{f(\mathbf{Y})}}_F\leq{L\Norm{\mathbf{X}-\mathbf{Y}}_F},
    \end{equation}
    in which $L$ is a constant and $\Norm{\cdot}_F$ is Frobenius norm, then
    \begin{equation}
    f(\mathbf{X})-f(\mathbf{Y})\geq{\trace\cBrac{\grad{f(\mathbf{Y})}\tp(\mathbf{X}-\mathbf{Y})}+\frac{1}{2L}\Norm{\grad{f(\mathbf{X})}-\grad{f(\mathbf{Y})}}_F^2}
    \end{equation}
\end{lemma}
\begin{proof}
    See Theorem 2.1.5 in \cite{Nesterov2018LecturesOC}.
\end{proof}

\begin{lemma}
    \label{lm:BoundedDecrease}
    Provided that Assumption~\ref{asm:LegalityCondition} is satisfied, then the following inequality holds for any $\mathbf{T}\in\mathbb{R}_+^M$ and any $\mathbf{D}_P$:
    \begin{equation}
    J(\mathbf{D}_P, \mathbf{T})-J(\mathbf{D}_P^*, \mathbf{T})\geq\frac{1}{4\sigma_P(\mathbf{T})}\Norm{\grad_{\mathbf{D}_P}J(\mathbf{D}_P, \mathbf{T})}_F^2
    \end{equation}
    where
    \begin{equation}
    \label{eq:MinimumCondition}
    \mathbf{D}_P^*=\argmin_{\mathbf{D}_P}{J(\mathbf{D}_P, \mathbf{T})},
    \end{equation}
    and $\sigma_P(\mathbf{T})$ is the largest singular value of $\mathbf{R}_{PP}(\mathbf{T})$.
\end{lemma}
\begin{proof}
    The gradient of $J(\mathbf{D}_{P}, \mathbf{T})$ with respect to $\mathbf{D}_P$ can be calculated as
    \begin{equation}
    \grad_{\mathbf{D}_P}J(\mathbf{D}_{P}, \mathbf{T})=2\mathbf{R}_{FP}\tp(\mathbf{T})\mathbf{D}_F+2\mathbf{R}_{PP}\tp(\mathbf{T})\mathbf{D}_P.
    \end{equation}
    The difference in gradient at $\mathbf{D}_P$ and $\mathbf{D}_P^*$ is
    \begin{equation}
    \label{eq:PartialGradientDifference}
    \Norm{\grad_{\mathbf{D}_P}J(\mathbf{D}_P, \mathbf{T})-\grad_{\mathbf{D}_P}J(\mathbf{D}_P^*, \mathbf{T})}_F=2\Norm{\mathbf{R}_{PP}\tp(\mathbf{T})\rBrac{\mathbf{D}_P-\mathbf{D}_P^*}}_F
    \end{equation}
    Assumption~\ref{asm:LegalityCondition} ensures that $\mathbf{R}_{PP}(\mathbf{T})$ is nonzero matrix, which means it has largest singular value $\sigma_P(\mathbf{T}) > 0$ for any $\mathbf{T}\in\mathbb{R}_+^M$. According to the basic property of spectral norm, we have
    \begin{equation}
    \label{eq:SpectralNormIneqaulity}
    \frac{\Norm{\mathbf{R}_{PP}\tp(\mathbf{T})\rBrac{\mathbf{D}_P-\mathbf{D}_P^*}}_F}{\Norm{\mathbf{D}_P-\mathbf{D}_P^*}_F}\leq\max_{\Norm{\mathbf{X}}_F=1}\Norm{\mathbf{R}_{PP}\tp(\mathbf{T})\mathbf{X}}_F=\sigma_P(\mathbf{T}).
    \end{equation}
    Combining (\ref{eq:PartialGradientDifference}) and (\ref{eq:SpectralNormIneqaulity}), we get
    \begin{equation}
    \Norm{\grad_{\mathbf{D}_P}J(\mathbf{D}_P, \mathbf{T})-\grad_{\mathbf{D}_P}J(\mathbf{D}_P^*, \mathbf{T})}_F\leq{2\sigma_P(\mathbf{T})}\Norm{\mathbf{D}_P-\mathbf{D}_P^*}_F.
    \end{equation}
    According to Lemma \ref{lm:PartialConvexity} and Lemma \ref{lm:LGradientConvexProperty}, if we substitute $f(\cdot)$ by $J(\cdot, \mathbf{T})$, together with the fact that (\ref{eq:MinimumCondition}) implies $\grad_{\mathbf{D}_P}J(\mathbf{D}_P^*, \mathbf{T})=\mathbf{0}$, the result follows.
\end{proof}

\begin{theorem}
    \label{thm:UnconstrainedGlobalConvergence}
    Consider the process in Algorithm~\ref{alg:UnconstrainedSpatialTemporalAM} started with any $\mathbf{D}_P^0$. Provided that Assumption~\ref{asm:LegalityCondition} is satisfied, then the inequality always holds for $K$-th iteration:
    \[
    \min_{0\leq{k}\leq{K}}{\norm{\grad{J(\mathbf{D}_P^k, \mathbf{T}^k)}}}_F^2\leq{M_c \frac{J(\mathbf{D}_P^0, \mathbf{T}^0)-J_c}{K}}
    \]
    where $M_c$ and $J_c$ are both constant.
\end{theorem}
\begin{proof}
    It is clear that the objective function is non-increasing in any iteration, i.e., for any $k\geq{0}$, we have
    \begin{equation}
    J(\mathbf{D}_P^k, \mathbf{T}^k)\geq{J(\mathbf{D}_P^{k+1}, \mathbf{T}^k)}\geq{J(\mathbf{D}_P^{k+1}, \mathbf{T}^{k+1})}.
    \end{equation}
    Moreover, the objective function is non-negative, which means $J(\mathbf{D}_P^k, \mathbf{T}^k)\geq{0}$ for for any $k\geq{0}$. Therefore,
    \begin{equation}
    \lim_{k\rightarrow+\infty}{J(\mathbf{D}_P^k, \mathbf{T}^k)}=J_c.
    \end{equation}
    Since $\mathbf{D}_P^{k+1}=\argmin_{\mathbf{D}_P}{J(\mathbf{D}_P, \mathbf{T}^k)}$, the following condition holds by Lemma \ref{lm:BoundedDecrease}:
    \begin{equation}
    \frac{\Norm{\grad_{\mathbf{D}_P}{J(\mathbf{D}_P^k, \mathbf{T}^k)}}_F^2}{4\sigma_{PP}(\mathbf{T}^k)}\leq{J(\mathbf{D}_P^k, \mathbf{T}^k)-J(\mathbf{D}_P^{k+1}, \mathbf{T}^k)}\leq{J(\mathbf{D}_P^k, \mathbf{T}^k)-J(\mathbf{D}_P^{k+1}, \mathbf{T}^{k+1})}.
    \end{equation}
    Notice that $\grad_{\mathbf{T}}J(\mathbf{D}_P^k, \mathbf{T}^k)=\mathbf{0}$ in each iteration, then
    \begin{equation}
    \Norm{\grad{J(\mathbf{D}_P^k, \mathbf{T}^k)}}_F=\Norm{\grad_{\mathbf{D}_P}J(\mathbf{D}_P^k, \mathbf{T}^k)}_F.
    \end{equation}
    Therefore,
    \begin{equation}
    \frac{\Norm{\grad{J(\mathbf{D}_P^k, \mathbf{T}^k)}}_F^2}{4\sigma_P(\mathbf{T}^k)}\leq{J(\mathbf{D}_P^k, \mathbf{T}^k)-J(\mathbf{D}_P^{k+1}, \mathbf{T}^{k+1})}.
    \end{equation}
    We simply let $\alpha=J(\mathbf{D}_P^0, \mathbf{T}^0)$, then $(\mathbf{D}_P^k, \mathbf{T}^k)\in\mathit{L}_{\alpha}^{-}(J)$ for all $k\geq{0}$. According to Assumption~\ref{asm:LegalityCondition}, $\mathit{L}_{\alpha}^{-}(J)$ is bounded and satisfies condition (\ref{eq:StrictPositiveness}). Then there exists positive constant $m_T$ and $M_T$ such that $\mathbf{T}^k\in[m_T,M_T]^{M}$ always holds for $k\geq{0}$. Consequently, $4\sigma_{P}(\mathbf{T}^k)$ is also upper bounded by a positive constant $M_c$. We have
    \begin{equation}
    \label{eq:LowerBoundedDecrease}
    \frac{\Norm{\grad{J(\mathbf{D}_P^k, \mathbf{T}^k)}}_F^2}{M_c}\leq\frac{\Norm{\grad{J(\mathbf{D}_P^k, \mathbf{T}^k)}}_F^2}{4\sigma_P(\mathbf{T}^k)}\leq{J(\mathbf{D}_P^k, \mathbf{T}^k)-J(\mathbf{D}_P^{k+1}, \mathbf{T}^{k+1})}.
    \end{equation}
    We sum it up for all $K$ iterations:
    \begin{equation}
    \frac{1}{M_c}\sum_{k=0}^{K}{\Norm{\grad{J(\mathbf{D}_P^k, \mathbf{T}^k)}}_F^2}\leq{J(\mathbf{D}_P^0, \mathbf{T}^0)-J(\mathbf{D}_P^{K+1}, \mathbf{T}^{K+1})}\leq{J(\mathbf{D}_P^0, \mathbf{T}^0)-J_c}
    \end{equation}
    Since the right hand side is bounded, we have
    \begin{equation}
    \lim_{K\rightarrow+\infty}{\grad{J(\mathbf{D}_P^K, \mathbf{T}^K)}}=\mathbf{0}.
    \end{equation}
    Taking the minimum of left hand side equals
    \begin{equation}
    \frac{\min\limits_{0\leq{k}\leq{K}}{\norm{\grad{J(\mathbf{D}_P^k, \mathbf{T}^k)}}}_F^2}{M_c/K}\leq{J(\mathbf{D}_P^0, \mathbf{T}^0)-J_c}.
    \end{equation}
    Rearranging gives the result.
\end{proof}

Theorem~\ref{thm:UnconstrainedGlobalConvergence} shows that, under no assumption on convexity, Algorithm~\ref{alg:UnconstrainedSpatialTemporalAM} shares the same global convergence rate $O(1/\sqrt{K})$ as that of gradient descent with the best step-size chosen in each iteration~\cite{Nesterov2018LecturesOC}. However, the best step-size is practically unavailable.
As a contrast, Algorithm~\ref{alg:UnconstrainedSpatialTemporalAM} does not involve any step-size choosing in each iteration.
Sub-problems (Eq.~\ref{eq:OptimizeD} and Eq.~\ref{eq:OptimizeT}) both can be solved exactly and efficiently due to their algebraic convenience. Therefore, Algorithm~\ref{alg:UnconstrainedSpatialTemporalAM} is faster than gradient-based methods in practice.

Although only convergence to stationary point is guaranteed, strict saddle points are theoretically and numerically unstable \cite{Lee2019FirstorderMA} for Algorithm~\ref{alg:UnconstrainedSpatialTemporalAM}, which is indeed a first-order method. Moreover, when the stationary point is a strict local minimum, we show that the convergence rate is faster than the general case in Theorem~\ref{thm:UnconstrainedGlobalConvergence}.

\begin{lemma}
    \label{lm:QuadraticRecurrence}
    Consider a positive sequence $\cBrac{\alpha_k}_{k\geq{0}}$ satisfying
    \begin{equation}
    \alpha_k-\alpha_{k+1}\geq{\gamma\alpha_k^2},~k=0,1,\cdots,
    \end{equation}
    where $\gamma>0$ is a constant. Then for all $k\geq{0}$,
    \begin{equation}
    \alpha_k\leq{\frac{1}{k\gamma+\alpha_0^{-1}}}.
    \end{equation}

\end{lemma}
\begin{proof}
    Apparently,
    \begin{equation}
    \frac{1}{\alpha_{k+1}}-\frac{1}{\alpha_k}\geq\frac{1}{\alpha_k-\gamma\alpha_k^2}-\frac{1}{\alpha_k}=\frac{\gamma}{1-\gamma\alpha_k}\geq\gamma,
    \end{equation}
    hence
    \begin{equation}
    \frac{1}{\alpha_{k}}-\frac{1}{\alpha_0}\geq{k\gamma}.
    \end{equation}
    Rearranging gives the result.
\end{proof}

\begin{theorem}
    \label{thm:UnconstrainedLocalConvergence}
    Provided that Assumption~\ref{asm:LegalityCondition} is satisfied, let $(\widehat{\mathbf{D}}_P, \widehat{\mathbf{T}})$ denote any strict local minimum of $J(\mathbf{D}_P, \mathbf{T})$ to which Algorithm~\ref{alg:UnconstrainedSpatialTemporalAM} converges, then there exist $K_c\in\mathbb{Z}_+$ and $\gamma\in\mathbb{R}_+$, such that
    \begin{equation}
    J(\mathbf{D}_P^K, \mathbf{T}^K)-J^*\leq{\frac{1}{\gamma(K-K_c)+(J(\mathbf{D}_P^{K_c}, \mathbf{T}^{K_c})-J^*)^{-1}}}
    \end{equation}
    for all $K\geq{K_c}$, where $J^*=J(\widehat{\mathbf{D}}_P, \widehat{\mathbf{T}})$.
\end{theorem}

\begin{proof}
    Define the neighborhood as
    \begin{equation}
    \mathcal{B}(r)=\cBrac{(\mathbf{D}_P, \mathbf{T})~\Big|~\norm{\mathbf{D}_P-\widehat{\mathbf{D}}_P}_F^2+\norm{\mathbf{T}-\widehat{\mathbf{T}}}^2\leq{r^2}}
    \end{equation}
    A strict local minimum $(\widehat{\mathbf{D}}_P, \widehat{\mathbf{T}})$ satisfies $\hessian{J(\widehat{\mathbf{D}}_P, \widehat{\mathbf{T}})}\succ\mathbf{0}$, then there exists $R_c\in\mathbb{R}_+$ such that $J$ is locally convex in the domain $\mathcal{B}(R_c)$. Moreover, there exists a positive integer $K_c$ such that $(\mathbf{D}_P^k, \mathbf{T}^k)\in\mathcal{B}(R_c)$ holds for all $k\geq{K_c}$, so we only consider $k\geq{K_c}$ hereafter. Due to the local convexity, we have
    \begin{equation}
    J(\mathbf{D}_P^k, \mathbf{T}^k)-J^*\leq\trace\cBrac{\grad_{\mathbf{D}_P}J(\mathbf{D}_P^k, \mathbf{T}^k)\tp{(\mathbf{D}_P^k-\widehat{\mathbf{D}}_P)}}+\grad_{\mathbf{T}}J(\mathbf{D}_P^k, \mathbf{T}^k)\tp(\mathbf{T}^k-\widehat{\mathbf{T}}).
    \end{equation}
    By applying Cauchy-Schwartz inequality on the right hand side, we have
    \begin{equation}
    (J(\mathbf{D}_P^k, \mathbf{T}^k)-J^*)^2\leq\Norm{\grad{J(\mathbf{D}_P^k, \mathbf{T}^k)}}_F^2\rBrac{\Norm{\mathbf{D}_P^k-\widehat{\mathbf{D}}_P}_F^2+\Norm{\mathbf{T}^k-\widehat{\mathbf{T}}}^2}.
    \end{equation}
    Notice that the distance between $(\mathbf{D}_P^k, \mathbf{T}^k)$ and the local minimum is upper-bounded by $R_c$, thus
    \begin{equation}
    \frac{(J(\mathbf{D}_P^k, \mathbf{T}^k)-J^*)^2}{R_c^2}\leq\Norm{\grad{J(\mathbf{D}_P^k, \mathbf{T}^k)}}_F^2.
    \end{equation}
    According to the inequality (\ref{eq:LowerBoundedDecrease}) deduced in the proof of Theorem \ref{thm:UnconstrainedGlobalConvergence}, we have
    \begin{equation}
    J(\mathbf{D}_P^k, \mathbf{T}^k)-J(\mathbf{D}_P^{k+1}, \mathbf{T}^{k+1})\geq\frac{\Norm{\grad{J(\mathbf{D}_P^k, \mathbf{T}^k)}}_F^2}{M_c},
    \end{equation}
    where $M_c$ is the upper bound of $4\sigma_P(\mathbf{T})$. Combining these two conditions, we get
    \begin{equation}
    J(\mathbf{D}_P^k, \mathbf{T}^k)-J(\mathbf{D}_P^{k+1}, \mathbf{T}^{k+1})\geq\frac{(J(\mathbf{D}_P^k, \mathbf{T}^k)-J^*)^2}{M_cR_c^2}.
    \end{equation}
    We apply Lemma \ref{lm:QuadraticRecurrence} by defining
    \begin{equation}
    \alpha_k:=J(\mathbf{D}_P^{k+K_c}, \mathbf{T}^{k+K_c})-J^*,~~~\gamma:=1/{M_cR_c^2},
    \end{equation}
    then the result follows.
\end{proof}

When Algorithm~\ref{alg:UnconstrainedSpatialTemporalAM} converges to a strict local minimum, the above theorem shows that the local convergence rate is $O(1/K)$. Note that it is possible to accelerate our method to attain the optimal rate $O(1/K^2)$ of first-order methods or use high-order methods to achieve a faster rate.

\bibliography{references}
\bibliographystyle{plain}

\end{document}